\documentclass[conference]{IEEEtran}
\IEEEoverridecommandlockouts
\usepackage{booktabs}  
\usepackage{multirow}  

\usepackage[dvipsnames,table,xcdraw]{xcolor}

\usepackage{cite}
\usepackage{amsmath,amssymb,amsfonts}
\usepackage{algorithmic}
\usepackage{graphicx}
\usepackage{textcomp}
\colorlet{lightlightgray}{lightgray!40}
\def\BibTeX{{\rm B\kern-.05em{\sc i\kern-.025em b}\kern-.08em
    T\kern-.1667em\lower.7ex\hbox{E}\kern-.125emX}}

\usepackage{graphicx}
\usepackage{upgreek}
\usepackage{bm}
\usepackage{soul}

\DeclareMathAlphabet{\mathcal}{OMS}{cmsy}{m}{n}
\usepackage{mathtools}

\usepackage{hyperref}
\usepackage{amsthm}

\newtheorem{theorem}{Theorem}

\newtheorem{lemma}{Lemma}

\usepackage{multirow}
\usepackage[normalem]{ulem}

\begin{document}

\title{Kernel-Based Anomaly Detection Using Generalized Hyperbolic Processes
}

\author{\IEEEauthorblockN{Pauline Bourigault\textsuperscript{1,2} \qquad Danilo P. Mandic\textsuperscript{2}}
\IEEEauthorblockA{\textit{\textsuperscript{1}Department of Computing, \textsuperscript{2}Department of Electronic and Electrical Engineering} \\
\textit{Imperial College London}\\
\{p.bourigault22, d.mandic\}@imperial.ac.uk}
}

\maketitle

\IEEEpubid{\begin{minipage}{\textwidth}\ \\[70pt]
\centering 
Copyright 2025 IEEE. Published in ICASSP 2025 – 2025 IEEE International Conference on Acoustics, Speech and Signal Processing (ICASSP), scheduled for 6-11 April 2025 in Hyderabad, India. Personal use of this material is permitted. However, permission to reprint/republish this material for advertising or promotional purposes or for creating new collective works for resale or redistribution to servers or lists, or to reuse any copyrighted component of this work in other works, must be obtained from the IEEE. Contact: Manager, Copyrights and Permissions / IEEE Service Center / 445 Hoes Lane / P.O. Box 1331 / Piscataway, NJ 08855-1331, USA. Telephone: + Intl. 908-562-3966.
\end{minipage}}

\IEEEpubidadjcol

\begin{abstract}
We present a novel approach to anomaly detection by integrating Generalized Hyperbolic (GH) processes into kernel-based methods. The GH distribution, known for its flexibility in modeling skewness, heavy tails, and kurtosis, helps to capture complex patterns in data that deviate from Gaussian assumptions. We propose a GH-based kernel function and utilize it within Kernel Density Estimation (KDE) and One-Class Support Vector Machines (OCSVM) to develop anomaly detection frameworks. Theoretical results confirmed the positive semi-definiteness and consistency of the GH-based kernel, ensuring its suitability for machine learning applications. Empirical evaluation on synthetic and real-world datasets showed that our method improves detection performance in scenarios involving heavy-tailed and asymmetric or imbalanced distributions. \href{https://github.com/paulinebourigault/GHKernelAnomalyDetect}{https://github.com/paulinebourigault/GHKernelAnomalyDetect}.
\end{abstract}

\begin{IEEEkeywords}
kernel, anomaly detection, generalized hyperbolic processes
\end{IEEEkeywords}

\section{Introduction}\label{sec:intro} Anomaly detection is the process of identifying data points or events that deviate significantly from the majority of observations. Its importance is noticeable in a wide range of application domains, from cybersecurity intrusion detection and fraud detection to healthcare diagnostics and manufacturing fault prevention \cite{Chandola2009, Aggarwal2013OutlierA}. Although numerous approaches have been developed for anomaly detection, kernel-based methods such as one-class Support Vector Machines (OCSVM) and Kernel Density Estimation (KDE) \cite{Parzen1962} have been particularly popular due to their theoretical soundness and tractability in high-dimensional spaces \cite{Schlkopf2001}.
Conventional kernel-based approaches generally employ Gaussian or polynomial kernels, which rely on assumptions of data distributions that can be overly simplistic. In practice, real-world datasets often exhibit highly non-Gaussian characteristics, including heavy-tailed, skewed, or multi-modal behaviors. Moreover, imbalanced class distributions—where anomalies constitute a small fraction of the data—can further degrade performance \cite{Zhao2019PyODAP, Chalapathy2019}. For instance, in network traffic analysis or financial time series, anomalous events often lie in the tails of complex distributions, making them difficult to detect with Gaussian-based kernels.
\begin{figure}[t!]
    \centering
    \includegraphics[width=\columnwidth]{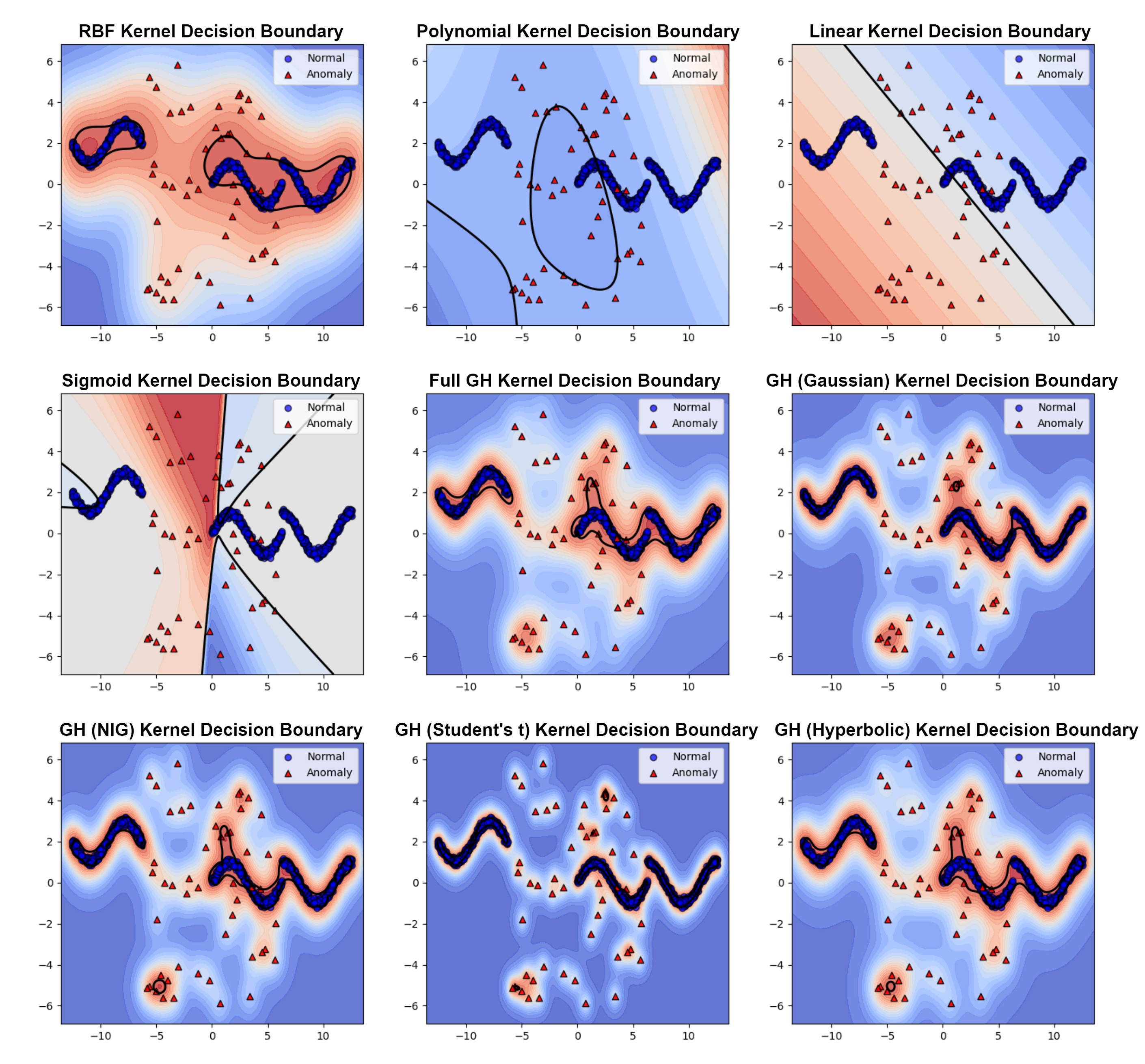}
    \caption{Decision boundaries for anomaly detection using various OCSVM kernels: standard and GH kernels. Blue circles denote normal data points, while red triangles represent anomalies. The black line represents the method's decision boundary. It visually illustrates the kernel's ability to distinguish between normal and anomalous data points. See Table \ref{tab:anomaly_detection_synthetic} and Fig. \ref{fig:enter-label} for details.}
    \label{fig:decbound}
\end{figure}

\textbf{Motivation and Contributions.} To address these limitations, we propose a novel kernel-based anomaly detection framework that leverages Generalized Hyperbolic (GH) distributions to capture heavy tails, skewness, and other complex distributional features common in real-world data. We (i) design a GH kernel that serves as a drop-in replacement for standard Gaussian kernels in kernel-based anomaly detection methods, (ii) establish the theoretical properties of the GH kernel (positive semi-definiteness, tail behavior, consistency), and (iii) demonstrate how to integrate this GH kernel into both OCSVM and KDE. Through experiments on synthetic and benchmark anomaly detection datasets (KDDCup99, ForestCover), we show that GH-based kernels outperform standard kernels, particularly in scenarios involving skewed or heavy-tailed distributions.

\section{Related Work}\label{sec:related}

Classical anomaly detection encompasses distance-based (e.g., \emph{k}NN), density-based (e.g., Local Outlier Factor), and tree-based approaches (e.g., Isolation Forest) \cite{IsolationForest, Aggarwal2013OutlierA}. Deep learning methods such as autoencoders and Generative Adversarial Networks (GANs) have also shown promise \cite{Chalapathy2019, schlegl2017unsupervised, Zhao2019PyODAP}, but often require extensive data and tuning, with less theoretical transparency than kernel-based models.
Among kernel-based approaches, the OCSVM \cite{Schlkopf2001} finds a boundary around normal data, labeling outliers as anomalies. Standard OCSVM generally assumes Gaussian or polynomial kernels, which may not capture heavy tails or skewness well. Alternatives such as Student’s \emph{t}-kernels and copula-based kernels \cite{NIPS2010_fc8001f8} improve flexibility, yet still struggle to jointly model skewness and heavy tails.
GH distributions \cite{Barndorff1977} offer a unified framework for heavy-tailed families (e.g., Student’s \emph{t}, normal-inverse Gaussian) and excel in modeling asymmetry and tail risk \cite{Cont2001}. Despite their suitability for non-Gaussian behaviors, GH-based methods remain underutilized in anomaly detection. Their ability to encode tail decay rates and skewness can alleviate the restrictive assumptions of purely Gaussian kernels.
GH-based frameworks have been explored mainly in areas such as clustering of skewed data \cite{browne2015mixture, Wei2019} and volatility modeling in finance \cite{Eberlein2002, Eberlein1995,prause1999generalized,takahashi2016volatility}. However, these do not address anomaly detection from a kernel-based perspective. This gap motivates our introduction of GH-based kernels for anomaly detection, enabling kernel-based methods such as OCSVM or KDE to better handle complex, non-Gaussian data.

\section{GH Kernel-Based Anomaly Detection}

\subsection{Kernel Derivation from Generalized Hyperbolic Processes}\label{sec:gh-kernel}

\textbf{Anomaly detection.}
Anomaly detection refers to the identification of rare events or outliers in data, which significantly deviate from the expected pattern. Let $X=\{x_1,x_2,\dots,x_n\}$ be a dataset where each $x_i \in \mathbb{R}^d$. The goal is to assign an anomaly score $S(x_i)$ to each point such that higher scores correspond to anomalies. In kernel-based methods, this is often achieved by learning a decision function or a density estimate that assigns lower values to points lying in low-density regions.

\textbf{GH Distribution.}
The GH distribution is defined by its characteristic function (or moment-generating function). Let $X$ be a random variable following a GH distribution $X \sim \text{GH}(\lambda , \alpha , \beta , \delta , \mu)$, where $\lambda$ controls the tail behavior, $\alpha$ is the scale parameter,
$\beta$ is the skewness parameter,
$\delta$ controls the kurtosis of the distribution, and
$\mu$ is the location parameter. The probability density function (PDF) of $X$ is given by
\begin{gather}
\begin{aligned}
    f_{\text{GH}}&( x \,| \,  \lambda, \alpha, \beta , \delta, \mu)  =  \\ & \frac{(\gamma / \delta)^\lambda}{\sqrt{2 \pi} K_{\lambda} (\delta \gamma)} e^{\beta(x-\mu)} \frac{K_{\lambda-1/2} \Bigl(\alpha \sqrt{\delta^2 + (x-\mu)^2}\Bigr)}{\Bigl( \sqrt{\delta^2 + (x-\mu)^2} /\alpha \Bigr)^{\lambda - 1/2}},
\end{aligned}
\end{gather}
where $\gamma = \sqrt{\alpha^2 - \beta^2}$ and $K_{\lambda}$ is the modified Bessel function of the second kind. The GH distribution generalizes several well-known distributions, including the Gaussian, Student's t, and NIG distributions, making it suitable for modeling data with complex distributional characteristics.

\textbf{GH Kernel Function.}
To leverage the GH distribution in anomaly detection, we define a kernel function based on the GH distribution that measures the similarity between two data points $x$ and $y$ as
\begin{equation}\label{eq:GH_kernel_expression}
    K_{\text{GH}(x,y)} = \int_{\mathbb{R}} f_{\text{GH}} (x-u) f_{\text{GH}}(y-u) \, du ,
\end{equation}
where $f_{\text{GH}}(z)$ is the GH distribution PDF. The kernel represents the similarity between $x$ and $y$ by comparing the likelihoods of these points under the GH distribution.
\begin{lemma}
The GH kernel $K_{\text{GH}}(x,y)$ is positive semi-definite (PSD) to satisfy Mercer's condition for appropriately chosen parameters $\lambda$, $\alpha$, $\beta$, $\delta$, $\mu$.
\end{lemma}
\begin{proof}
    To prove the PSD property, we need to show that for any finite set of points $\{x_1,x_2,\dots , x_n\} \in \mathbb{R}^d$, the Gram matrix $[K_{\text{GH}}(x_i,x_j)]$ is positive semi-definite, meaning
    \begin{equation}
        \sum_{i=1}^{n} \sum_{j=1}^{n} c_i c_j K_{\text{GH}} (x_i,x_j) \geq 0 \quad \forall \{c_i\} \in \mathbb{R}^n .
    \end{equation}
    Using the integral representation of the GH kernel and its convolution form, we express the kernel as
    \begin{equation}
        K_{\text{GH}} (x_i,x_j) = \int_{\mathbb{R}^d} f_{\text{GH}} (x_i - u) f_{\text{GH}} (x_j - u) \, du.
    \end{equation}
    For a set of coefficients $\{c_i\} \in \mathbb{R}^n$, the quadratic form now becomes
    \begin{gather}
    \begin{aligned}
        \sum_{i=1}^{n} \sum_{j=1}^{n} c_i & c_j K_{\text{GH}} (x_i,x_j) =  \\ & \sum_{i=1}^{n} \sum_{j=1}^{n} c_i c_j \int_{\mathbb{R}^d} f_{\text{GH}} (x_i - u) f_{\text{GH}} (x_j - u) \, du.
    \end{aligned}
    \end{gather}
    Since $f_{\text{GH}}(z)$ is a PDF and thus non-negative, the kernel $K_{\text{GH}}(x_i,x_j)$ forms a valid inner product in the function space induced by the GH process. Therefore, the Gram matrix is positive semi-definite, and $K_{\text{GH}}(x,y)$ is a valid kernel.
\end{proof}
\begin{theorem}
    For large values of $| x - y |$, the GH kernel $K_{\text{GH}}(x,y)$ asymptotically decays as an exponential function of the distance $| x - y |$, controlled by the parameters $\alpha$ and $\beta$.
\end{theorem}
\begin{proof}
    For large $| x - y |$, the GH PDF exhibits tail behavior similar to that of a heavy-tailed distribution. Specifically, the tail of the GH distribution decays exponentially as $e^{- \alpha | x - \mu | }$. Therefore, the kernel function behaves as
\begin{equation}
    K_{\text{GH}}(x,y) \sim e^{-\frac{1}{2} (x- y)^{T} \sum_{\text{GH}}^{-1}(x-y)} \; \text{for}\; | x - y | \rightarrow \infty,
\end{equation}
where $\sum_{\text{GH}}$ is a covariance matrix derived from the second-order moments of the GH process. Thus, the GH kernel asymptotically behaves like a Gaussian kernel for large distances, but with parameters influenced by the tail behavior of the GH distribution.
\end{proof}

\subsection{Integration of GH Kernel with Support Vector Machines}\label{sec:gh-svm}
\textbf{GH-OCSVM.} In the context of anomaly detection, the OCSVM aims to find a hyperplane in the kernel-induced feature space that separates the normal data from the origin, with the anomaly score $S(x)$ being the distance from the hyperplane. The objective function for the OCSVM is
\begin{equation}
    \min_{\pmb{w}, \xi, \rho} \frac{1}{2} \| \pmb{w} \|^2 + \frac{1}{\nu n} \sum_{i=1}^{n} \xi_i - \rho ,
\end{equation}
subject to $\quad \pmb{w} \,\phi(x_i) \geq \rho - \xi_i, \quad \xi_i \geq 0, \quad i=1, \dots, n,$
\vspace{0.2cm}
\\
where $\pmb{w}$ is the normal vector in the high-dimensional feature space, $\phi(x_i)$ is the feature mapping induced by the kernel function $K(x_i,x_j)$, $\xi_i$ are slack variables, $\rho$ is the offset, and $\nu$ controls the number of support vectors.

\begin{theorem}
    Given the GH kernel $K_{\text{GH}}(x_i,x_j)$, the dual form of the OCSVM optimization problem is
    \begin{equation}
        \max_{\alpha_i} \sum_{i=1}^n \alpha_i K_{\text{GH}}(x_i,x_i) - \frac{1}{2} \sum_{i=1}^n \sum_{j=1}^n \alpha_i \alpha_j K_{\text{GH}}(x_i,x_j), 
    \end{equation}
    subject to
    $ \quad 0 \leq \alpha_i \leq \frac{1}{\nu n}, \quad \sum_{i=1}^n \alpha_i =1$.
\end{theorem}
\begin{proof}
    The primal optimization problem for OCSVM is
    \begin{equation}
        \min_{\pmb{w},\xi_i,\rho} \frac{1}{2} \| \pmb{w} \|^2 + \frac{1}{\nu n} \sum_{i=1}^n \xi_i - \rho .
    \end{equation}
    Using Lagrange multipliers $\alpha_i \geq 0$ for each constraint $\pmb{w}\phi(x_i) \geq \rho - \xi_i$, the Lagrangian is
    \begin{equation}
        \mathcal{L}(\pmb{w},\rho,\alpha_i) = \frac{1}{2} \| \pmb{w} \|^2 - \sum_{i=1}^n \alpha_i \Bigl(\pmb{w}\, \phi(x_i) - \rho + \xi_i \Bigr). 
    \end{equation}
    Taking the dual form of the problem and substituting the GH kernel $K_{\text{GH}}(x_i,x_j) = \phi (x_i) \phi (x_j)$, we obtain the dual problem
    \begin{equation}
        \max_{\alpha_i} \sum_{i=1}^n \alpha_i K_{\text{GH}}(x_i,x_i) - \frac{1}{2} \sum_{i=1}^n \sum_{j=1}^n \alpha_i \alpha_j  K_{\text{GH}}(x_i,x_j),
    \end{equation}
    with the constraints $0 \leq \alpha_i \leq \frac{1}{\nu n}$ and $\sum_{i=1}^n \alpha_i =1$. This establishes the optimality condition for the dual formulation of OCSVM with GH kernel. 
\end{proof}

\subsection{Kernel-Based Density Estimation using GH Processes}\label{sec:gh-kde}
We next develop a kernel-based anomaly detection method that utilizes GH processes and is distinct from the standard SVM method. Instead, it leverages concepts from kernel density estimation and spectral methods.

\textbf{Kernel Density Estimation (KDE).}
KDE is a non-parametric method to estimate the PDF of a random variable. It evaluates the underlying density of the data at each point by averaging the contributions of all points using a kernel function. Low-density regions are typically flagged as anomalous. Given a set of observations $\{x_1,x_2,\dots,x_n\}$, KDE defines the estimated density $\hat{f}(x)$ at any point $x \in \mathbb{R}^d$ as
\begin{equation}
    \hat{f}(x) = \frac{1}{nh^d} \sum_{i=1}^n K \biggl(\frac{x-x_i}{h} \biggr),
\end{equation}
where $K(\cdot)$ is a kernel function,
$h>0$ is a smoothing parameter (or bandwidth), and $d$ is the dimension of the data. In our case, the kernel function $K(\cdot)$ will be derived from the GH process.

\textbf{GH-KDE.}
The GH-KDE is a kernel density estimator where the kernel function is based on the GH process. The GH-KDE estimate $\hat{f}_{GH}(x)$ of the density at any point $x$ is given by
\begin{equation}
\hat{f}_{GH}(x) = \frac{1}{n h^d} \sum_{i=1}^{n} K_{GH} \left( \frac{x - x_i}{h} \right),
\end{equation}
where $K_{GH}(\cdot)$ is the positive semi-definite GH kernel function defined in Eq. (\ref{eq:GH_kernel_expression}), and $h$ is the bandwidth. The bandwidth $h$ controls the level of smoothing, with smaller $h$ leading to sharper peaks in the estimated density.

\begin{theorem}\label{thm:consistencyGH-KDE}
    The GH-KDE estimator $\hat{f}_{GH}(x)$ is a consistent estimator of the true density $f(x)$ as the number of data points $n \to \infty$ and the bandwidth $h \to 0$ at an appropriate rate.
\end{theorem}
\begin{proof}
    We must prove that the GH-KDE estimator $\hat{f}_{GH}(x)$ converges in probability to the true density $f(x)$ as $n \to \infty$. The consistency of KDE methods relies on two factors. As $n \to \infty$ and $h \to 0$, the bias of the estimator should vanish. As $n \to \infty$, the variance of the estimator should also decrease.

The bias of the GH-KDE is given by the expected difference between the estimated density and the true density, as
\begin{equation}
\text{Bias}\Bigl(\hat{f}_{GH}(x)\Bigr) = \mathbb{E}\Bigl[\hat{f}_{GH}(x)\Bigr] - f(x).
\end{equation}
Expanding $\hat{f}_{GH}(x)$ in a Taylor series around $x$, we obtain that the bias term is of order $O(h^2)$, assuming the smoothness of the true density $f(x)$, that is
\begin{equation}
\text{Bias}\Bigl(\hat{f}_{GH}(x)\Bigr) \sim O(h^2) \quad \text{as} \quad  h \to 0.
\end{equation}
The variance of the estimator is
\begin{equation}
\text{Var}\Bigl(\hat{f}_{GH}(x)\Bigr) = \frac{1}{n h^d} \int_{\mathbb{R}^d} K_{GH}^2\left( \frac{x - u}{h} \right) f(u) \, du.
\end{equation}
Using the fact that $K_{GH}(x, y)$ is a valid kernel with finite support, the variance term is of order $O\left( \frac{1}{n h^d} \right)$.
Thus, as $n \to \infty$ and $h \to 0$ such that $n h^d \to \infty$, the variance goes to zero. Therefore, combining the bias and variance terms, we conclude that the GH-KDE is a consistent estimator, that is
\begin{equation}
\hat{f}_{GH}(x) \overset{p}{\to} f(x) \quad \text{as} \quad n \to \infty \quad \text{and} \quad h \to 0.
\end{equation}
\end{proof}
\textbf{Anomaly Score Based on GH-KDE.}
For anomaly detection, the estimated density $\hat{f}_{GH}(x)$ can be used to assign an anomaly score to each point. The anomaly score for a point $x$ is defined as
\begin{equation}
S(x) = -\log \Bigl(\hat{f}_{GH}(x)\Bigr),
\end{equation}
where $\hat{f}_{GH}(x)$ is the GH-KDE estimate of the density at $x$.

\begin{theorem}
    As the number of data points $n \to \infty$, the anomaly score $S(x)$ asymptotically converges to the negative log-likelihood of the true density $f(x)$, i.e.,
\begin{equation}
S(x) \to -\log \Bigl(f(x)\Bigr) \quad \text{as} \quad n \to \infty.
\end{equation}
\end{theorem}
\begin{proof}
    Using the consistency of the GH-KDE estimator $\hat{f}_{GH}(x)$ from Theorem \ref{thm:consistencyGH-KDE}, we know that
\begin{equation}
\hat{f}_{GH}(x) \overset{p}{\to} f(x) \quad \text{as} \quad n \to \infty.
\end{equation}
Taking the logarithm of both sides, we have
\begin{equation}
\log \Bigl(\hat{f}_{GH}(x)\Bigr) \to \log\Bigl(f(x)\Bigr) \quad \text{as} \quad n \to \infty.
\end{equation}
Therefore, the anomaly score converges to the negative log-likelihood of the true density, given by
\begin{equation}
S(x) = -\log \Bigl(\hat{f}_{GH}(x)\Bigr) \to -\log \Bigl(f(x) \Bigr) \quad \text{as} \quad n \to \infty.
\end{equation}
This shows that the anomaly score based on GH-KDE provides an asymptotically accurate measure of the rarity of a point $x$ in the dataset.
\end{proof}

\begin{table}[t!]
\caption{\textbf{Anomaly Detection Performance on Synthetic Dataset} \\(mean $\pm$ SD over 10 seeds for methods with stochasticity or random sampling). Gridsearch used for finetuning.}
\vspace{-0.25cm}
\label{tab:anomaly_detection_synthetic}
\renewcommand{\arraystretch}{0.8} 
\centering
\resizebox{\columnwidth}{!}{%
\begin{tabular}{@{}lrrr@{}}
\toprule
\textbf{Model} &
  \textbf{\begin{tabular}[c]{@{}l@{}}AUC-ROC ($\uparrow$)\end{tabular}} &
  \textbf{\begin{tabular}[c]{@{}l@{}}Train. Time (s) ($\downarrow$)\end{tabular}} &
  \textbf{\begin{tabular}[c]{@{}l@{}}\#Hyperparams\end{tabular}} \\ \midrule
\multicolumn{4}{l}{\textbf{OCSVM Kernels}} \\
RBF                           & $0.976$  & $0.008$  &  $2$\\
Polynomial                    & $0.870$  & $0.02$  &  $4$\\
Linear                        & $0.547$  & \textcolor{orange}{{$\bm{0.001}$}}  &  $1$\\
Sigmoid                       & $0.437$  & $0.007$  &  $3$\\ 

\rowcolor{lightgray!40}
Full GH Kernel                & \textcolor{orange}{$\bm{0.997}$}  & $1.60$  &  $5$\\
\rowcolor{lightgray!40}
GH Kernel (Gaussian)          & $0.981$  & $1.15$  &  $3$\\
\rowcolor{lightgray!40}
GH Kernel (NIG)               & $0.989$  & $1.18$  &  $4$\\
\rowcolor{lightgray!40}
GH Kernel (Student's t)       & $0.984 $ & $1.67$  & $ 4$\\
\rowcolor{lightgray!40}
GH Kernel (Hyperbolic)        & $0.992$ & $1.62 $ &  $4$\\ \midrule
\multicolumn{4}{l}{\textbf{KDE Kernels}} \\
Gaussian                      & $0.203$  & $0.004$   &  $1$\\
Epanechnikov                  & $0.500$  & $0.004 $  &  $1$\\
Tophat                        & $0.500$  & $0.004$   &  $1$\\
Exponential                   & $0.055 $ & $0.004$  &  $1$\\ 

\rowcolor{lightgray!40}
Full GH Kernel                & $0.956$  & $4.84$  &  $5$\\
\rowcolor{lightgray!40}
GH Kernel (Gaussian)          & $0.966$  & $4.83$  &  $2$\\
\rowcolor{lightgray!40}
GH Kernel (NIG)               & $0.967$  & $4.80$  &  $3$\\
\rowcolor{lightgray!40}
GH Kernel (Student's t)       & $0.968$  & $4.82$  &  $3$\\
\rowcolor{lightgray!40}
GH Kernel (Hyperbolic)        & $0.950$  & $4.87$  &  $3$\\ \midrule
Vanilla Autoencoder           & $0.220 \pm 0.16$  & $1.06$  &  $\sim 2 - 7$\\
Variational Autoencoder       & $0.461 \pm 0.04$  & $4.21$  &  $\sim 3 - 8$\\ 
\begin{tabular}[c]{@{}l@{}}Deep Autoencoding GMM \cite{zong2018deep}\end{tabular}  & $0.972 \pm 0.03$ & $9.73$ & $\sim 11 - 15$
\\ 
\begin{tabular}[c]{@{}l@{}}Deep SVDD \cite{pmlr-v80-ruff18a}\end{tabular}  & $0.984 \pm 0.02$ & $7.68$ & $\sim 10 - 11$
\\ 
\begin{tabular}[c]{@{}l@{}}Isolation Forest \cite{IsolationForest}\end{tabular}  & $0.210 \pm 0.01$ & $9.51$ & $\sim 5 - 6$
\\ 
\begin{tabular}[c]{@{}l@{}}One-Class Neural Network \cite{One-ClassNN}\end{tabular}  & $0.709 \pm 0.02$ & $5.28$ &  $\sim 8 - 13$ \\
\begin{tabular}[c]{@{}l@{}}Memory-augmented Deep AE \cite{gong2019memorizing} 
\end{tabular} & $0.731 \pm 0.05$ & $27.62$ &  $\sim 8 - 14$
\\ \bottomrule
\end{tabular}%
}
\end{table}
\begin{figure}
    \centering
    \includegraphics[width=0.6\columnwidth]{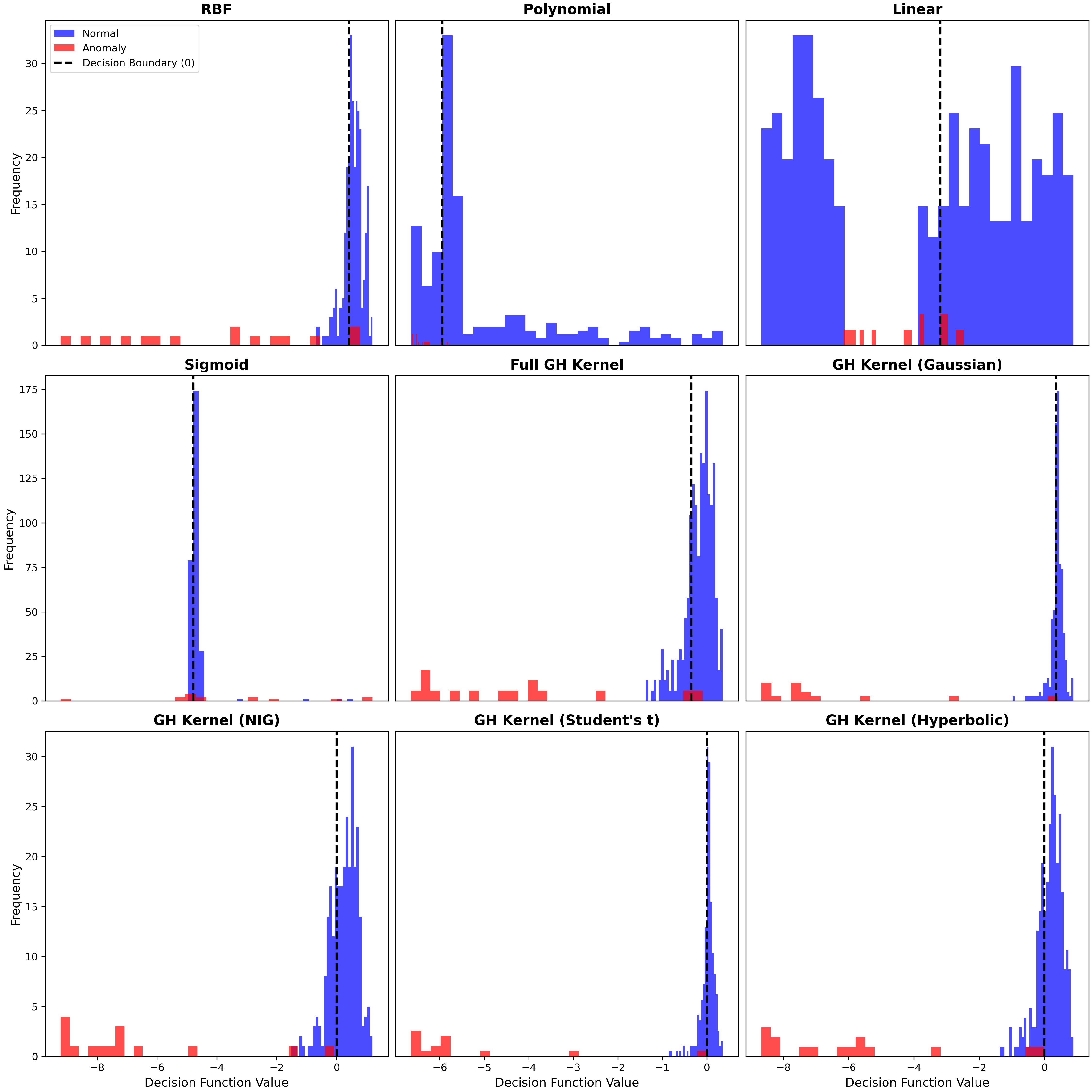}
    \caption{Histograms of decision function values for the OCSVM kernels from Fig. \ref{fig:decbound} and Tab. \ref{tab:anomaly_detection_synthetic}, highlighting the separation between normal data (blue) and anomalies (red). The dashed black line represents the decision boundary at 0.}
    \label{fig:enter-label}
\end{figure}
\section{Empirical Validation}\label{sec:experiments}
For validation, we used a synthetic dataset and two common anomaly detection datasets which are KDDCup99 \cite{KDDCUP99} and \href{https://www.kaggle.com/datasets/uciml/forest-cover-type-dataset}{ForestCover} from the UCI Machine Learning Repository. We compared kernel functions for OCSVM, including the Radial Basis Function (RBF), Polynomial, Linear, and
Sigmoid. These were tested against kernels derived from the GH distributions, that include a general case covering various GH distributions (Full GH), one reduced to the Gaussian distribution, one using the NIG distribution, one that captures heavy-tailed behavior (Student's t), and one hyperbolic capturing different tail behaviors. The inclusion of GH distributions allows us to test different tail behaviors and skewness in the data distributions. For KDE, we tested the standard Gaussian, Epanechnikov, Tophat, and Exponential kernels. We also implemented KDE using the GH kernel
variants to provide a direct comparison
with OCSVM. We included other popular supervised and unsupervised anomaly detection techniques \cite{zong2018deep,pmlr-v80-ruff18a,IsolationForest, One-ClassNN, gong2019memorizing}. Among the metrics reported, precision and recall are particularly important for imbalanced datasets, where the anomalous class is underrepresented. We also report CPU training time and the number of hyperparameters per method.

In Table \ref{tab:anomaly_detection_synthetic} and Figures \ref{fig:decbound} and \ref{fig:enter-label}, we aim to generate data points on a set of manifolds for the normal class, representing structured patterns, and to scatter anomalies as sparse, unstructured points. This mimics real-world settings where normal data follows structured distributions (e.g., sensor readings, time-series, etc.) while anomalies break that structure. Kernels like Full GH and GH variants show clear separation near the decision boundary, indicating strong anomaly detection performance. In contrast, the Linear and Sigmoid kernels displayed significant overlap, reflecting poorer separation and lower detection capability (Fig.\ref{fig:enter-label}).
In Table \ref{tab:anomaly_detection}, the KDDCup99 dataset originates from network intrusion detection, with multiple classes of normal and anomalous behavior. The ForestCover dataset contains features describing forest cover types, with one class treated as normal and another as anomalous.
\begin{table}[t!]
\caption{\textbf{Anomaly Detection Performance on KDDCup99 \cite{KDDCUP99} and \href{https://www.kaggle.com/datasets/uciml/forest-cover-type-dataset}{ForestCover}}.}
\vspace{-0.25cm}
\label{tab:anomaly_detection}
\renewcommand{\arraystretch}{0.8} 
\centering
\resizebox{\columnwidth}{!}{%
\begin{tabular}{@{}lcccccc@{}}
\toprule
\textbf{Model} 
  & \multicolumn{3}{c}{\textbf{\uline{KDDCup99}}} 
  & \multicolumn{3}{c}{\uline{\textbf{ForestCover}}} \cr 
   & \textbf{Acc.($\%$)} 
   & \textbf{Rec.($\%$)} 
   & \textbf{Prec.($\%$)} 
   & \textbf{Acc.($\%$)} 
   & \textbf{Rec.($\%$)} 
   & \textbf{Prec.($\%$)} \cr 
\midrule
\multicolumn{7}{l}{\textbf{OCSVM Kernels}} \cr
RBF                          & $96.5$  & $92.8$  & $93.2$  & $96.8$  & $91.5$  & $88.9$ \cr
Polynomial                   & $90.3$  & $88.1$  & $85.4$  & $94.2$  & $87.1$  & $84.7$ \cr
Linear                       & $88.7$  & $80.5$  & $76.2$  & $92.6$  & $81.8$  & $75.3$ \cr
Sigmoid                      & $65.2$  & $50.3$  & $45.8$  & $64.1$  & $51.4$  & $48.2$ \cr
\rowcolor{lightgray!40}
Full GH Kernel               & $98.7$  & \textcolor{orange}{$\bm{95.2}$}  & \textcolor{orange}{$\bm{96.7}$}  & $97.1$  & $94.6$  & $93.3$ \cr
\rowcolor{lightgray!40}
GH Kernel (Gaussian)         & $97.3$  & $93.4$  & $94.8$  & $97.6$  & $93.8$  & $92.3$ \cr
\rowcolor{lightgray!40}
GH Kernel (NIG)              & $97.6$  & $94.1$  & $95.5$  & $97.8$  & $94.5$  & $93.1$ \cr
\rowcolor{lightgray!40}
GH Kernel (Student's t)      & $97.9$  & $94.7$  & $96.1$  & $97.7$  & $94.3$  & $92.8$ \cr
\rowcolor{lightgray!40}
GH Kernel (Hyperbolic)       & $97.8$  & $94.6$  & $96.0$  & $97.9$  & \textcolor{orange}{$\bm{94.9}$}  & $93.5$ \cr 
\midrule
\multicolumn{7}{l}{\textbf{KDE Kernels}} \cr
Gaussian                     & $52.3$  & $43.2$  & $38.7$  & $52.7$  & $39.6$  & $35.2$ \cr
Epanechnikov                 & $68.9$  & $55.4$  & $50.6$  & $70.3$  & $54.8$  & $50.1$ \cr
Tophat                       & $67.8$  & $54.9$  & $49.2$  & $69.2$  & $53.9$  & $48.7$ \cr
Exponential                  & $35.4$  & $30.1$  & $25.8$  & $42.3$  & $30.8$  & $26.4$ \cr
\rowcolor{lightgray!40}
Full GH Kernel               & $92.1$  & $86.4$  & $85.0$  & $94.2$  & $88.3$  & $87.5$ \cr
\rowcolor{lightgray!40}
GH Kernel (Gaussian)         & $91.5$  & $85.3$  & $84.1$  & $92.8$  & $86.1$  & $85.0$ \cr
\rowcolor{lightgray!40}
GH Kernel (NIG)              & $92.7$  & $86.8$  & $85.9$  & $93.5$  & $87.4$  & $86.3$ \cr
\rowcolor{lightgray!40}
GH Kernel (Student's t)      & $93.1$  & $87.2$  & $86.5$  & $94.1$  & $88.0$  & $87.1$ \cr
\rowcolor{lightgray!40}
GH Kernel (Hyperbolic)       & $92.9$  & $87.0$  & $86.3$  & $93.8$  & $87.7$  & $86.8$ \cr
\midrule
Vanilla Autoencoder          & $57.3$  & $40.2$  & $45.1$  & $59.4$  & $40.7$  & $46.5$ \cr
Variational Autoencoder      & $62.5$  & $50.7$  & $53.8$  & $68.2$  & $54.3$  & $56.1$ \cr
Deep Autoencoding GMM \cite{zong2018deep} & $97.1$  & $92.3$  & $93.9$  & $97.5$  & $93.5$  & $91.2$ \cr
Deep SVDD \cite{pmlr-v80-ruff18a} & \textcolor{orange}{$\bm{98.2}$} & $94.9$ & $95.8$ & \textcolor{orange}{$\bm{98.0}$} & $94.8$ & \textcolor{orange}{$\bm{93.6}$} \cr
Isolation Forest \cite{IsolationForest} & $55.6$  & $38.6$  & $41.2$  & $58.7$  & $39.3$  & $43.8$ \cr
One-Class Neural Network \cite{One-ClassNN} & $85.7$  & $78.2$  & $80.9$  & $85.1$  & $79.4$  & $77.6$ \cr
Memory-augmented Deep AE \cite{gong2019memorizing} & $89.2$  & $82.3$  & $83.7$  & $88.3$  & $82.7$  & $81.4$ \cr
\bottomrule
\end{tabular}%
}
\end{table}
Here, the results further strengthen our proposed framework in terms of performance and interpretability. 

Overall, incorporating GH-based kernels into both OCSVM and KDE provides flexibility to capture complex data distributions, particularly for imbalanced anomaly detection tasks. Yet, these methods require more careful tuning and training time than standard kernels.

\bibliographystyle{IEEEtran}
\bibliography{IEEEabrv,paper}

\end{document}